\documentclass[11pt]{article}
\usepackage[margin=1in]{geometry}
\usepackage[T1]{fontenc}
\usepackage{lmodern}
\usepackage{amsmath,amssymb,amsthm}
\usepackage{graphicx}
\usepackage{booktabs}
\usepackage[round,authoryear]{natbib}
\usepackage{placeins}
\usepackage{microtype}
\usepackage[hidelinks]{hyperref}

\hypersetup{
  pdftitle={A Structural Threshold in Decision Capacity Governs Collapse in Self-Play Reinforcement Learning},
  pdfauthor={Arahan Kujur},
  pdfsubject={Multi-agent reinforcement learning; self-play; robustness; game theory},
  pdfkeywords={multi-agent reinforcement learning, self-play, robustness, action-space perturbation, game theory}
}

\newtheorem{proposition}{Proposition}
\newtheorem{corollary}{Corollary}

\title{A Structural Threshold in Decision Capacity Governs\\Collapse in Self-Play Reinforcement Learning}

\author{
  Arahan Kujur\\
  Independent Researcher\\
  \texttt{kujurarahan@gmail.com}
}
\date{}

\begin{document}
\maketitle

\begin{abstract}
We show that a threshold in decision capacity determines whether self-play reinforcement learning agents collapse under asymmetric rule perturbations. Across poker variants (Kuhn, Leduc, Leduc-4), matrix games (Matching Pennies), a dice game (Liar's Dice; 24{,}576 info sets), and six learning algorithms (Q-Learning, SARSA, REINFORCE, PPO, DQN, NFSP), eliminating all positive-reach contingent decisions causes rapid convergence to a deterministic exploitation attractor (DEA)---a fixed point at near-maximal loss. Preserving even a single positive-reach contingent decision point prevents this collapse. A frozen baseline and fixed-opponent control confirm the mechanism is co-adaptation under constraint, not the perturbation itself. The phenomenon is timing-invariant, fully reversible upon action restoration, and intensifies under function approximation. These results establish a practically sharp threshold induced by the discontinuity at $\text{CAC}_w = 0$, with severity scaling continuously via reach-weighted capacity in all tested domains.
\end{abstract}

\section{Introduction}

Multi-agent reinforcement learning (MARL) agents trained through self-play have achieved superhuman performance in complex games~\citep{silver2018general,brown2019superhuman,vinyals2019grandmaster}, yet their robustness to structural changes in the environment remains poorly understood. Prior work has focused primarily on adversarial perturbations to observations or rewards~\citep{gleave2020adversarial}, opponent modelling under distribution shift~\citep{foerster2018learning}, or training stability in population-based methods~\citep{balduzzi2019open}. Structural changes to the \emph{action space}---where an agent permanently loses access to certain actions---remain largely unexplored.

Such perturbations arise naturally in practice. In robotics, hardware failures may disable actuators, eliminating actions from an agent's repertoire mid-deployment. In financial trading, regulatory changes can restrict previously available strategies. In multi-agent software systems, API deprecations may remove action endpoints. Understanding how self-play agents respond to these asymmetric capability losses is essential for deploying RL systems in environments where the action space is not guaranteed to remain static.

We study this question in discrete, imperfect-information games by deterministically removing one player's ability to bet or raise at specified subsets of decision nodes. We discover a structural failure mode with a precise threshold: the agent's \emph{contingent action capacity} (CAC)---the number of information sets at which it retains more than one legal action---governs whether self-play dynamics collapse or stabilise. We report unweighted CAC for interpretability; the reach-weighted variant $\text{CAC}_w$ discounts rarely reached decision points and is the quantity that aligns with best-response discontinuities.

Our central finding is a pronounced threshold effect. When $\text{CAC}_w$ drops to zero (every positive-reach decision is forced), adaptive self-play agents converge to a \emph{deterministic exploitation attractor} (DEA)---the unique optimal policy of the single-agent MDP induced by the forced player's policy $\sigma_0^f$. Its stability arises from the absence of counterfactual branching under $\text{CAC}_w = 0$: the opponent faces a stationary, deterministic environment with a unique best response, and $\varepsilon$-greedy perturbations bound deviation from the theoretical minimum. The static fact that zero contingency reduces the game to best response against $\sigma_0^f$ is expected; the contribution is the observed learning transition from collapse at zero capacity to near-Nash stability when even one positive-reach decision point is preserved. A frozen baseline and fixed-opponent control isolate continued co-adaptation as the mechanism. Our contributions:
\begin{itemize}
  \item We identify a pronounced reach-weighted CAC threshold governing collapse in self-play RL and provide formal propositions characterising the zero-contingency fixed point.
  \item We isolate co-adaptation under constraint as the mechanism via frozen baseline and fixed-opponent comparisons.
  \item We demonstrate that collapse persists and intensifies under function approximation (DQN).
  \item We replicate across eight game variants (1--24{,}576 info sets), six algorithms, multiple perturbation schedules, and show full reversibility. Boundary conditions (IPD, Liar's Dice, cooperative games) sharpen the threshold definition.
\end{itemize}

\section{Related Work}

\paragraph{Solving imperfect-information games.}
Counterfactual regret minimisation (CFR)~\citep{zinkevich2007regret} and its variants have solved increasingly large poker games, from heads-up limit hold'em~\citep{bowling2015heads} to no-limit variants via DeepStack~\citep{moravcik2017deepstack}, Libratus~\citep{brown2018superhuman}, and Pluribus~\citep{brown2019superhuman}. These systems assume a fixed game structure. We study what happens when the game structure changes asymmetrically after training.

\paragraph{Self-play reinforcement learning.}
Self-play has driven advances from TD-Gammon~\citep{tesauro1994td} through AlphaZero~\citep{silver2018general} and AlphaStar~\citep{vinyals2019grandmaster}. Neural Fictitious Self-Play (NFSP)~\citep{heinrich2016deep} combines RL with average-strategy tracking. Policy-Space Response Oracles (PSRO)~\citep{lanctot2017unified} maintain diverse strategy populations. However, self-play dynamics can be unstable---agents may cycle, overfit to their own weaknesses, or exhibit non-transitive behaviour~\citep{balduzzi2019open,lanctot2019openspiel}. Our work identifies a distinct failure mode: co-adaptation-driven collapse under asymmetric action-space constraints.

\paragraph{Robustness in multi-agent RL.}
The MARL robustness literature encompasses adversarial policies~\citep{gleave2020adversarial}, opponent-learning awareness~\citep{foerster2018learning}, and distributional robustness~\citep{zhang2021multi}. These typically perturb observations, rewards, or opponent behaviour. We perturb the \emph{action space} itself---a structural change that eliminates decision points rather than adding noise. This reveals a qualitative threshold effect that continuous perturbations cannot produce.

\paragraph{Action masking and constrained RL.}
Invalid action masking is standard practice in game AI~\citep{huang2022closer}, and constrained MDPs formalise action restrictions~\citep{altman1999constrained}. Robust MDPs address uncertain transitions~\citep{iyengar2005robust,nilim2005robust}, and stochastic action-set MDPs model settings where available actions vary~\citep{boutilier2018planning}. However, the dynamic consequences of \emph{mid-training} action removal under self-play have not been studied.

\paragraph{Self-play stabilisation.}
Opponent-shaping methods such as LOLA~\citep{foerster2018learning} and SOS~\citep{letcher2019stable} stabilise self-play by accounting for the opponent's learning dynamics. Population-based methods including PSRO~\citep{lanctot2017unified} and $\alpha$-PSRO~\citep{muller2020generalized} maintain diverse strategy populations to avoid cycling. We test PSRO empirically and find it \emph{mitigates} but does not eliminate collapse (\S\ref{sec:mechanism}).

\paragraph{Exploitability and game-theoretic evaluation.}
Exploitability---the gap between an agent's value and the Nash equilibrium value---is the standard evaluation metric in computational game theory~\citep{johanson2007computing,timbers2022approximate}. We connect our threshold effect to the best-response structure of reduced games via formal propositions (Section~5).

\section{Background}

\paragraph{Kuhn Poker.}
Three cards ($J < Q < K$), two actions (pass, bet), ante~1. Nash value for P0: $-1/18 \approx -0.056$. 12 information sets.

\paragraph{Leduc Poker.}
Six cards ($J, Q, K \times 2$ suits), three actions (fold, check/call, raise), two rounds, fixed-limit betting. Nash P0 value: $\approx -0.087$. 288 information sets. Our CFR implementation converges to $-0.0866$.

\paragraph{Leduc-4 Poker.}
A Leduc variant with four ranks ($J, Q, K, A$) and three suits (12 cards). Same rules as Leduc. 504 information sets. Nash P0 value $\approx -0.096$.

\paragraph{Liar's Dice (1 die).}
Two players each roll one six-sided die (private). Players alternate claiming a minimum count of a face value across both dice; claims must strictly increase. A player may challenge instead of claiming. On challenge, the claim is verified---if true the challenger loses, otherwise the claimer loses ($\pm 1$). 13 actions (12~claims + challenge). 24{,}576 information sets. Our CFR implementation yields Nash P0 value $\approx -0.076$.

\paragraph{Coordination Game.}
Two agents simultaneously choose among 3 actions to match a randomly revealed target. Payoff +1 if both match, 0 otherwise. 10 rounds/episode. \emph{Cooperative}, not competitive---the first non-zero-sum domain tested.

\paragraph{Matching Pennies.}
Two players, two actions (heads, tails), simultaneous single-shot. Nash: 50/50 mixed, value~0. 1 information set per player.

\paragraph{Agents.}
\textbf{CFR}: Nash equilibrium via full-tree CFR, frozen post-training.
\textbf{Q-Learning}: Tabular, $\varepsilon$-greedy ($\varepsilon = 0.15$), MC terminal updates.
\textbf{QL-Frozen}: Q-Learning frozen at perturbation (Q-table and $\varepsilon$ fixed).
\textbf{DQN}: 2-layer MLP (64 hidden), experience replay, target network.
\textbf{SARSA/REINFORCE}: On-policy tabular variants.

\section{Methodology}

Each experiment runs 20 seeds of 20{,}000 episodes (50{,}000 for DQN). Perturbation applied to Player~0 at the midpoint.

\paragraph{Perturbation protocol.}
Unless stated otherwise, perturbations are deterministic: we fix a named action label or action set and remove it from Player~0's legal set after the midpoint. In poker domains this removes bet/raise; in Matching Pennies it removes heads; in IPD it removes cooperate; in Liar's Dice boundary experiments it removes high/all claims or forces the lowest legal action. Scope is controlled by the experiment: root-only removal affects only Player~0's first decision, while full removal affects every Player~0 information set where the named action is legal. Thus the Kuhn CAC sweep compares full bet removal (zero contingency), root-only bet removal (one residual call/fold decision), and the unperturbed control. Stochastic masking, when used, is a separate ablation over whether this fixed mask is active in an episode, not a random choice of which action to delete.

\paragraph{Contingent action capacity (CAC).}
The number of reachable information sets at which the perturbed agent retains $>1$ legal action. All CAC-based experiments use the unweighted count as a proxy; $\text{CAC}_w$ (reach-weighted, Section~5) is the governing theoretical quantity. All observed CAC thresholds correspond to $\text{CAC}_w > 0$ vs.\ $\text{CAC}_w = 0$ regimes; the unweighted count serves as an interpretable proxy that aligns with the reach-weighted distinction in all tested games.

\paragraph{Self-play setup.}
A single agent plays both roles, with Q-values indexed by player-specific information states. This is equivalent to independent self-play in our zero-sum setting because each player's policy is implicitly determined by separate information-state entries. We verify this empirically: separate-agent experiments produce identical results (Appendix~\ref{app:separate}).

\paragraph{Statistical analysis.}
Paired $t$-tests across seeds, bootstrap 95\% CIs (10{,}000 resamples), Cohen's~$d$. Due to low across-seed variance under paired evaluation, effect sizes are large and should be interpreted as indicating direction and reliability.

\paragraph{Normalization.}
Cross-game comparison uses $(r - r_{\min}) / (r_{\max} - r_{\min}) \in [0,1]$. This normalization aids comparison across reward scales but obscures absolute exploitability; we report exact exploitability where computable (Kuhn, Leduc) to complement normalized rewards. Normalized comparisons align with exploitability trends in both games where both metrics are available.

\section{Theory}

We formalise the threshold effect in two-player zero-sum extensive-form games. The propositions below apply to the zero-sum case; Section~\ref{sec:boundaries} presents empirical evidence that cooperative and mixed-motive settings produce a qualitatively different (bounded degradation) response. Extending the formal treatment to general-sum games is left to future work. Let $G = (H, P, A, \mathcal{I}, u)$ be a two-player zero-sum game with history set $H$, player function $P$, action sets $A$, information partition $\mathcal{I}$, and utility $u$.

\paragraph{Contingent action capacity.}
For a reduced game $G'$, the unweighted contingent action capacity of Player~$i$ is
\begin{equation}
\text{CAC}(P_i) = \sum_{h \in \mathcal{I}_i} \mathbf{1}[|A(h)| \geq 2].
\label{eq:cac}
\end{equation}
This counts information sets at which the player still has a real choice after perturbation.

\paragraph{Reach-weighted contingent action capacity.}
For a strategy profile $\sigma$, the \emph{reach probability} of information set $h$ is $\rho^\sigma(h) = \prod_{h' \sqsubset h} \sigma_{P(h')}(h')$, the product of action probabilities on the path to $h$. We define the \emph{reach-weighted CAC}:
\begin{equation}
\text{CAC}_w(P_i) = \sum_{h \in \mathcal{I}_i} \rho^{\sigma}(h) \cdot \mathbf{1}[|A(h)| \geq 2]
\label{eq:cac-w}
\end{equation}
This refines the unweighted count by discounting information sets that are rarely reached. The unweighted CAC used in experiments is an upper bound: $\text{CAC} \geq \text{CAC}_w$.

\begin{proposition}[Zero-contingency exploitation]
\label{prop:zero}
Let $G'$ be the reduced game with $\text{CAC}(P_0) = 0$ (equivalently $\text{CAC}_w(P_0) = 0$). Then $P_0$'s value in $G'$ is:
\[
v_0(G') = -\max_{\pi_1 \in \Pi_1} \sum_{z \in Z} u_1(z) \cdot \rho^{\sigma_0^f, \pi_1}(z)
\]
where $\sigma_0^f$ is $P_0$'s forced (deterministic) strategy. Moreover, $P_1$'s best response is pure and computable in $O(|Z|)$ time.
\end{proposition}

\begin{proof}
Under zero contingency, $\sigma_0^f$ is the unique strategy for $P_0$. The game reduces to a single-player MDP for $P_1$ with a deterministic environment. The optimal policy is a pure best response computable by backward induction over $P_1$'s decision nodes, with complexity linear in the number of terminal histories reachable under $\sigma_0^f$.
\end{proof}

\begin{proposition}[Residual contingency bound]
\label{prop:residual}
Suppose $P_0$ retains at least one information set $h^* \in \mathcal{I}_0$ with $|A(h^*)| \geq 2$ and reach $\rho^{\sigma}(h^*) > 0$. Let $v_0^*$ be $P_0$'s minimax value in the subgame rooted at $h^*$, and let $v_0^{\text{forced}}$ be the value under the forced policy at all other nodes. Then:
\[
v_0(G') \geq \rho^{\sigma}(h^*) \cdot v_0^* + (1 - \rho^{\sigma}(h^*)) \cdot v_0^{\text{forced}}
\]
In particular, $v_0(G') > v_0^{\text{zero-cont}}$ whenever $v_0^* > v_0^{\text{forced}}$ and $\rho^{\sigma}(h^*) > 0$. The improvement is proportional to the reach of the retained node:
\[
\delta(h^*) = \rho^{\sigma}(h^*) \cdot (v_0^* - v_0^{\text{forced}})
\]
\end{proposition}

\begin{proof}
By linearity of expectation over reach probabilities, $P_0$'s value decomposes into the contribution from histories passing through $h^*$ (where $P_0$ plays optimally) and those not passing through $h^*$ (where $P_0$ is forced). The bound follows directly. When $h^*$ is reached with probability 1 (as in Kuhn root-only removal where all games pass through the ``pb'' node), $\delta(h^*) = v_0^* - v_0^{\text{forced}}$ and the full minimax value is recovered.
\end{proof}

\begin{corollary}
The transition from $\text{CAC}_w > 0$ to $\text{CAC}_w = 0$ qualitatively changes the best-response structure: from a game with strategic interaction (requiring mixed or adaptive responses, with value bounded by $\delta(h^*)$) to one with a trivially computable pure best response. A single retained decision point with positive reach is sufficient to prevent collapse; a retained decision with zero reach provides no protection.
\end{corollary}

\begin{proposition}[DEA as a fixed point of self-play dynamics]
\label{prop:dea}
Consider tabular Q-learning self-play under zero contingency ($\text{CAC}_w = 0$). Let $Q_1(s,a)$ denote $P_1$'s Q-values. Under standard assumptions ($\varepsilon > 0$, $\alpha_t \to 0$, $\sum \alpha_t = \infty$), $P_1$'s Q-values converge to the unique fixed point $Q_1^* = Q^{\text{BR}(\sigma_0^f)}$---the optimal Q-function for a single-agent MDP defined by $P_0$'s forced policy $\sigma_0^f$. The resulting joint policy profile $(\sigma_0^f, \pi_1^*)$ is a Nash equilibrium of the reduced game $G'$ and constitutes the DEA.
\end{proposition}

\begin{proof}
Under $\text{CAC}_w = 0$, $P_0$'s policy is fixed and deterministic at every reachable information set. From $P_1$'s perspective, the environment is a stationary MDP with deterministic transitions (conditioned on the card deal). Standard Q-learning convergence guarantees~\citep{tesauro1994td} apply: with decaying learning rate and persistent exploration, $Q_1 \to Q_1^*$ almost surely. Since $P_0$ cannot deviate from $\sigma_0^f$, the profile $(\sigma_0^f, \pi_1^*)$ satisfies the Nash equilibrium conditions of $G'$. The fixed point is unique because the reduced MDP has a unique optimal Q-function. The $\varepsilon$-greedy policy induced by $Q_1^*$ places probability $\geq 1 - \varepsilon$ on the best response at each information set, yielding $v_0 \geq V_{\min} + \varepsilon_{\text{floor}}$ where $\varepsilon_{\text{floor}} = O(\varepsilon)$.
\end{proof}

\section{Experiments}

Results are organised as: phenomenon (\S\ref{sec:phenomenon}), threshold (\S\ref{sec:threshold}), mechanism (\S\ref{sec:mechanism}), generalisation (\S\ref{sec:generalisation}), boundary conditions (\S\ref{sec:boundaries}), and dynamics (\S\ref{sec:dynamics}). Additional experiments appear in the Appendix.

\FloatBarrier
\subsection{The Phenomenon}
\label{sec:phenomenon}

\paragraph{Zero contingency (CAC $= 0$).}
Bet removed from P0 at all nodes in Kuhn. P0 is forced to check then fold.

\begin{table}[ht!]
\centering
\caption{Kuhn, zero contingency. Player~0 reward (20 seeds, 95\% CIs).}
\label{tab:kuhn-full}
\begin{tabular}{lcccc}
\toprule
Agent & Pre & Post & $p$ & $d$ \\
\midrule
CFR        & $-0.060$ {\small$[-0.067, -0.054]$} & $-0.221$ {\small$[-0.225, -0.216]$} & ${<}0.0001$ & $-6.9$ \\
Q-Learning & $-0.041$ {\small$[-0.050, -0.032]$} & $\mathbf{-0.926}$ {\small$[-0.928, -0.924]$} & ${<}0.0001$ & $-42.0$ \\
\bottomrule
\end{tabular}
\end{table}

Q-Learning converges to the DEA at $-0.926$ (normalized: $0.27$; $d = -42.0$) within a mean of four episodes. CFR drops to $-0.22$, bounded by its static opponent.

\begin{figure}[ht!]
\centering
\includegraphics[width=0.85\linewidth]{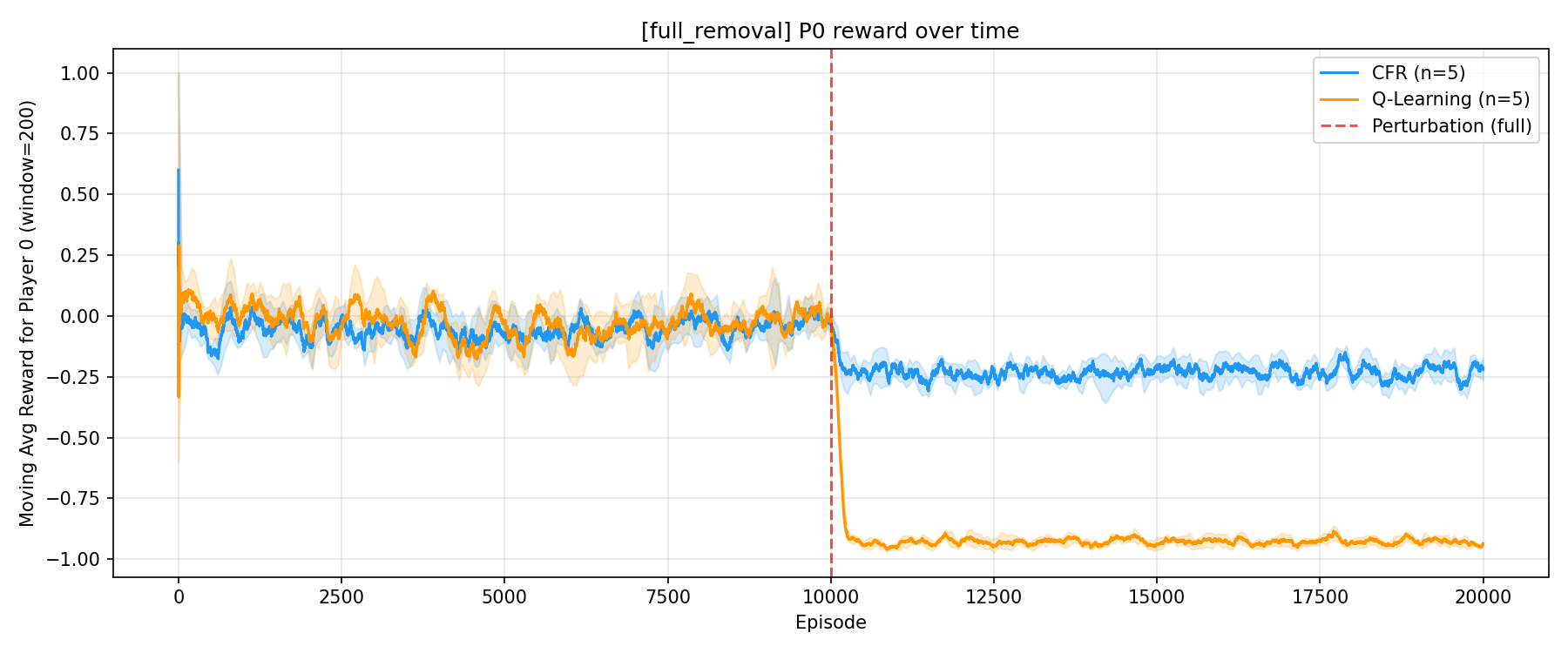}
\caption{Zero contingency (Kuhn): Q-Learning converges to the DEA within four episodes, illustrating the transition to deterministic exploitation.}
\label{fig:kuhn-full}
\end{figure}

\paragraph{Residual contingency (CAC $= 1$).}
Bet removed at root only; P0 retains call/fold. Q-Learning drops modestly ($\Delta = -0.024$, $d = -0.8$) then stabilises near Nash (Table~\ref{tab:kuhn-root}).

\begin{table}[ht!]
\centering
\caption{Kuhn, residual contingency (CAC $= 1$). 20 seeds.}
\label{tab:kuhn-root}
\begin{tabular}{lcccc}
\toprule
Agent & Pre & Post & $p$ & $d$ \\
\midrule
CFR        & $-0.060$ & $-0.053$ & $0.20$ & $+0.3$ \\
Q-Learning & $-0.041$ & $-0.065$ & $0.002$ & $-0.8$ \\
\bottomrule
\end{tabular}
\end{table}

\FloatBarrier
\subsection{The Threshold}
\label{sec:threshold}

\begin{table}[ht!]
\centering
\caption{CAC sweep (Kuhn): the discontinuity at CAC $= 0 \to 1$.}
\label{tab:capacity}
\begin{tabular}{clccc}
\toprule
CAC & Description & CFR & QL & QL Norm. \\
\midrule
0 & Zero-contingency          & $-0.221$ & $\mathbf{-0.926}$ & $0.27$ \\
1 & Residual contingency      & $-0.053$ & $-0.065$ & $0.48$ \\
2 & Full (control)            & $-0.053$ & $-0.034$ & $0.49$ \\
\bottomrule
\end{tabular}
\end{table}

The jump from CAC $= 0$ to $1$ is $\Delta = 0.86$ ($+0.21$ normalized). The jump from 1 to 2 is marginal ($\Delta = 0.03$). Figure~\ref{fig:capacity-threshold} shows the discontinuity on a normalized scale.

\begin{figure}[ht!]
\centering
\includegraphics[width=0.65\linewidth]{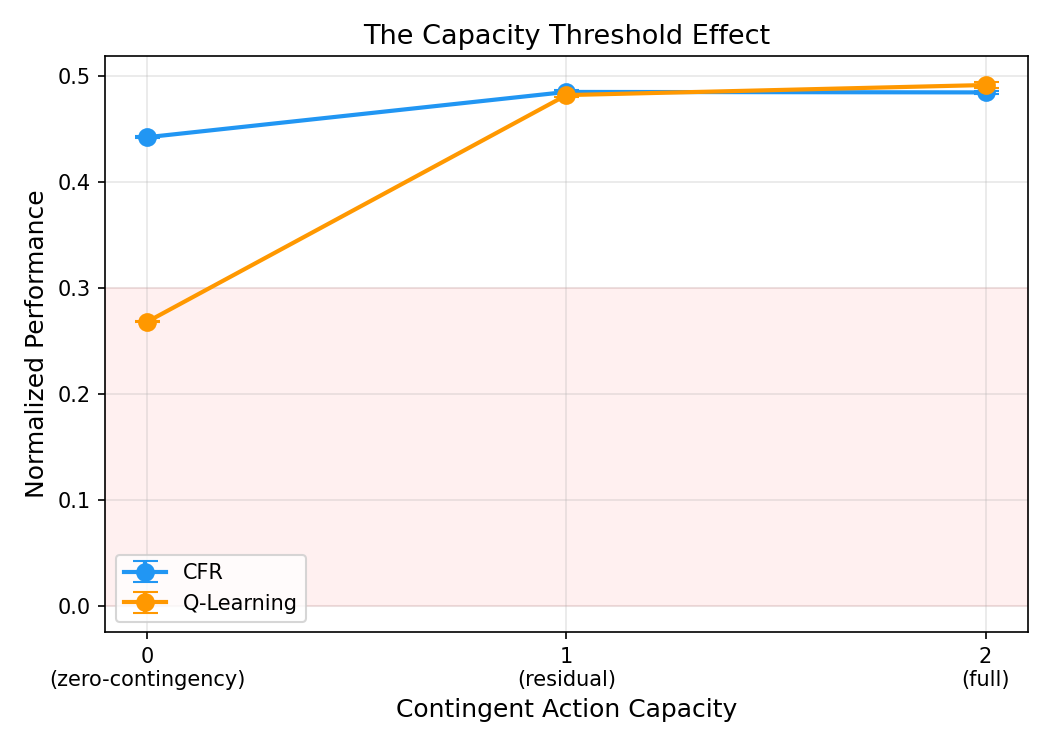}
\caption{The CAC threshold (normalized). The discontinuity at 0$\to$1 confirms a structural threshold.}
\label{fig:capacity-threshold}
\end{figure}

\FloatBarrier
\subsection{The Mechanism}
\label{sec:mechanism}

\paragraph{Frozen baseline.}
QL-Frozen avoids the DEA ($-0.141$ vs $-0.927$; $\Delta = -0.787$, $p = 0.004$, $d = -2.6$; 5 seeds). Co-adaptation---not the constraint---drives collapse.

\paragraph{Fixed opponent.}
Against a static Nash opponent (5 seeds), Q-Learning drops only to $-0.228$ (identical to CFR). Under self-play the same condition produces $-0.926$ (20 seeds). Co-adaptation is necessary for catastrophic collapse.

\paragraph{Population-based training (PSRO).}
Under PSRO with a 5-policy population, P0's post-perturbation reward improves to $-0.418 \pm 0.062$ (3 seeds)---substantially better than self-play ($-0.927$) but still degraded from pre-perturbation ($\approx -0.05$). The diverse population prevents the opponent from fully specialising against P0's forced policy, but cannot recover the strategic loss from zero contingency. PSRO \emph{mitigates} but does not \emph{eliminate} collapse: the structural vulnerability remains, but its exploitation is bounded by population diversity. With larger populations, P0's post-perturbation reward improves monotonically ($-0.275$ at pop.\ 3; $-0.598$ at pop.\ 15; all vs.\ $-0.927$ under self-play). Population diversity bounds the degree of exploitation but cannot reintroduce eliminated strategic dimensions: the structural vulnerability at CAC $= 0$ persists regardless of opponent diversity, though its severity is modulated.

\begin{figure}[ht!]
\centering
\includegraphics[width=0.85\linewidth]{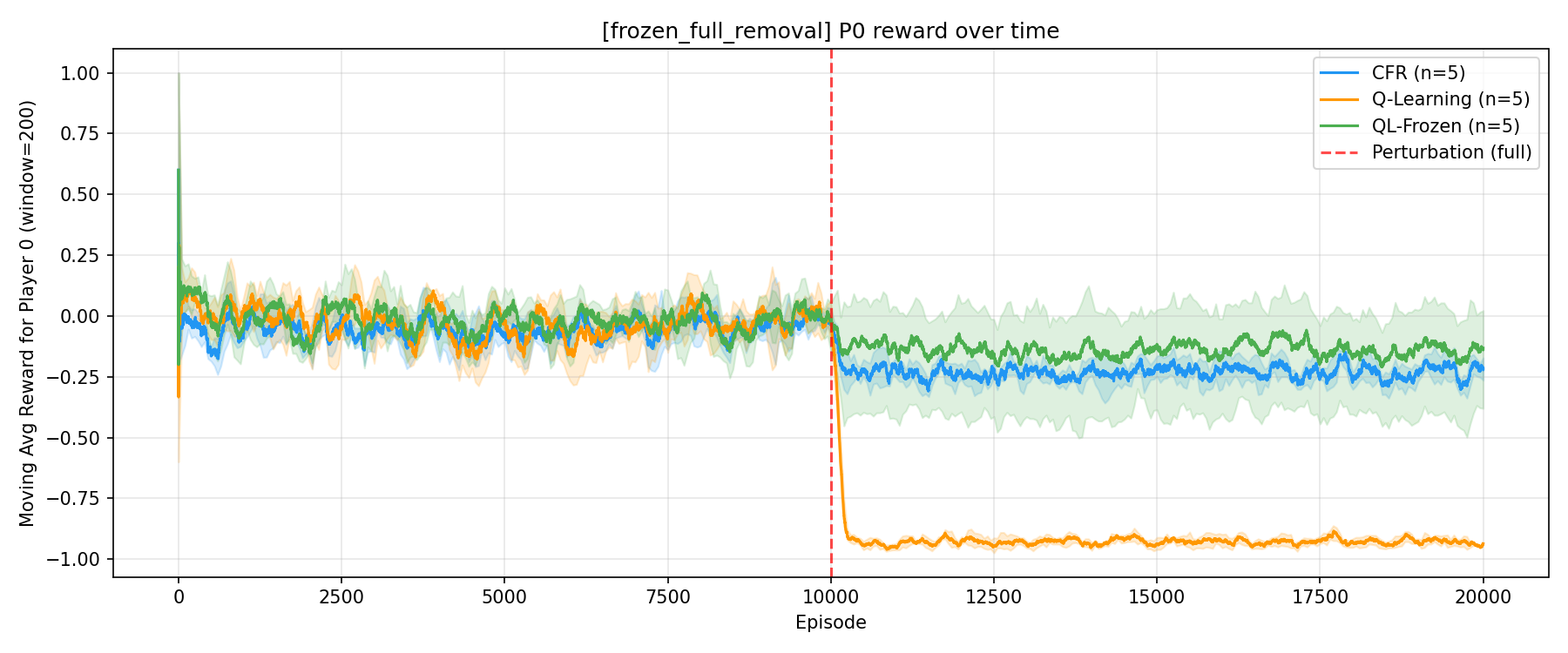}
\caption{QL-Frozen avoids the DEA, isolating co-adaptation as the mechanism.}
\label{fig:frozen}
\end{figure}

\FloatBarrier
\subsection{Generalisation}
\label{sec:generalisation}

\paragraph{Algorithm invariance.}

\begin{table}[ht!]
\centering
\caption{Algorithm invariance under zero contingency (Kuhn, 5 seeds).}
\label{tab:algorithms}
\begin{tabular}{lccc}
\toprule
Algorithm & Type & Post & $d$ \\
\midrule
Q-Learning & Tabular & $\mathbf{-0.927}$ & $-66.1$ \\
SARSA      & Tabular & $\mathbf{-0.927}$ & $-66.1$ \\
REINFORCE  & Tabular & $\mathbf{-0.500}$ & $-39.6$ \\
DQN        & Neural  & $\mathbf{-0.994}$ & $-24.4$ \\
PPO        & Tabular & $\mathbf{-0.500}$ & $-39.6$ \\
NFSP       & Tabular & $\mathbf{-0.505}$ & -- \\
\bottomrule
\end{tabular}
\end{table}

All algorithms collapse. DQN reaches $-0.994$---\emph{more} severe than tabular agents. PPO and REINFORCE both reach $-0.500$, partially protected by their softmax policy parameterisation. Notably, NFSP---which maintains an explicit average strategy via supervised learning, designed specifically for imperfect-information games---also collapses to $-0.505$. This demonstrates that structural constraints override regret-minimisation guarantees: NFSP's best-response component still learns to exploit the forced opponent, and the average strategy cannot compensate when the action space itself has collapsed, as there is no historical action diversity to average over when P0's every decision is forced. To isolate the $\varepsilon$-schedule effect, we ran DQN with fixed $\varepsilon = 0.15$ (no decay): post-perturbation reward improved modestly to $-0.923$, confirming that the deeper collapse under decaying $\varepsilon$ is driven by reduced exploration ($\varepsilon \to 0.01$), but the structural vulnerability persists regardless. Figure~\ref{fig:dqn-analysis} provides mechanistic insight: after perturbation, DQN's policy entropy drops to near zero and Q-value gaps spike, confirming rapid convergence to a deterministic policy.

\begin{figure}[ht!]
\centering
\includegraphics[width=0.85\linewidth]{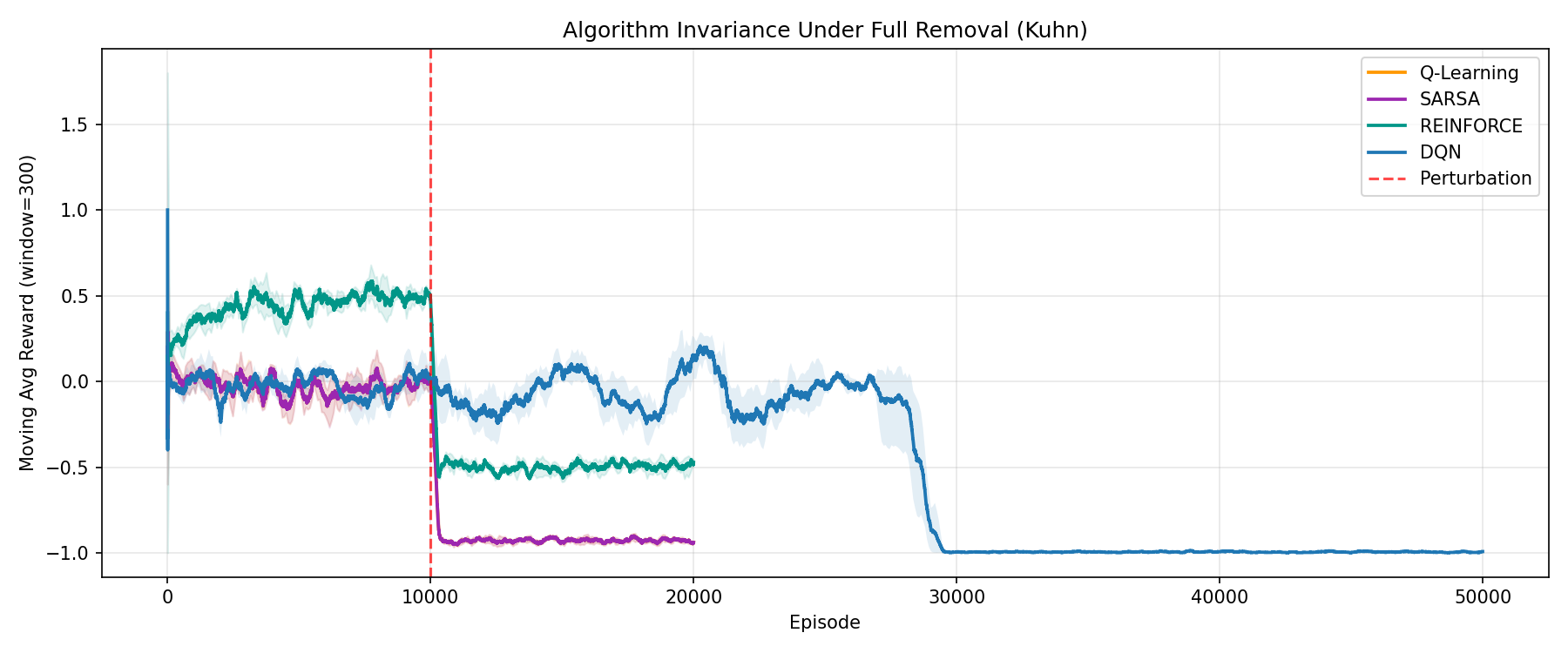}
\caption{All four algorithms converge toward the DEA. DQN collapses deepest.}
\label{fig:algo-overlay}
\end{figure}

\begin{figure}[ht!]
\centering
\includegraphics[width=0.95\linewidth]{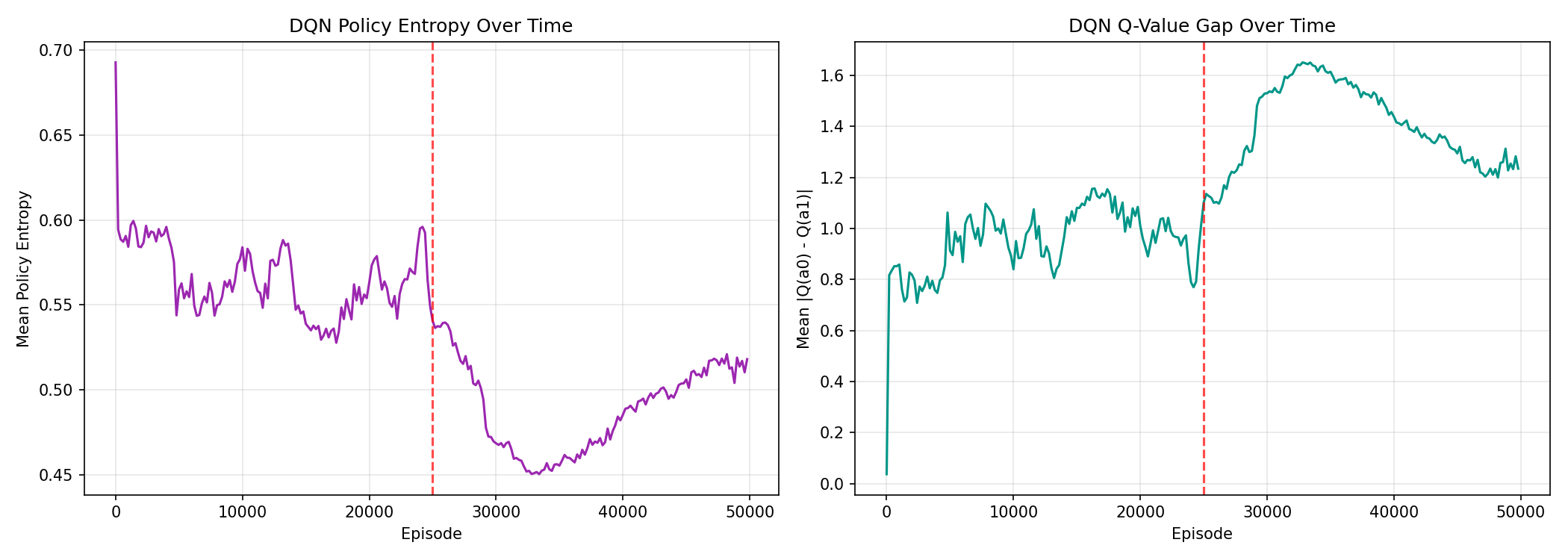}
\caption{DQN neural analysis under zero contingency (Kuhn). Left: policy entropy drops to near zero post-perturbation, confirming deterministic policy convergence. Right: Q-value gap ($|Q(a_0) - Q(a_1)|$) spikes, reflecting the network's growing certainty in the forced-fold action.}
\label{fig:dqn-analysis}
\end{figure}

\paragraph{Cross-game replication.}

\begin{table}[ht!]
\centering
\caption{Cross-game normalized collapse severity.}
\label{tab:cross-game}
\begin{tabular}{lccl}
\toprule
Game & QL Post & Info sets & Normalized \\
\midrule
Matching Pennies & $-0.851$ & 1    & $0.07$ \\
Kuhn Poker       & $-0.926$ & 12   & $0.27$ \\
Leduc Poker      & $-0.252$ & 288  & $0.49$ \\
Leduc-4 Poker    & $-0.185$ & $504$ & $0.49$ \\
Liar's Dice (1d) & $-0.032$ & $24{,}576$ & $0.48^\dagger$ \\
Liar's Dice (1d, true zero) & $\mathbf{-0.524}$ & $24{,}576$ & $0.24$ \\
Coordination     & $+1.44^{\ddagger}$ & -- & -- \\
Liar's Dice (2d) & $+0.008$ & $200{,}000{+}$ & $0.50^\dagger$ \\
\bottomrule
\end{tabular}

{\small $^\dagger$Challenge-only retains strategic value (\S\ref{sec:boundaries}). $^\ddagger$Cooperative game; forced single action reduces coordination but reward remains positive ($p\!=\!0.001$).}
\end{table}

Collapse holds across poker and matrix games (1--504 info sets). Liar's Dice at both scales (24{,}576 and $200{,}000{+}$ info sets) shows no collapse because challenge-only play retains strategic flexibility---a boundary condition consistent with Proposition~\ref{prop:residual}. Severity scales inversely with residual action options: Matching Pennies (most severe, 0.07) has zero residual options; Leduc variants (least severe, 0.49) retain fold/check-call. Under residual contingency, no game collapses.

\FloatBarrier
\subsection{Boundary Conditions}
\label{sec:boundaries}

\paragraph{IPD: when perturbation aligns with equilibrium.}
Removing ``cooperate'' in the Iterated Prisoner's Dilemma produces no collapse (post $= +1.12$). IPD Nash is always-defect; removing cooperate pushes P0 \emph{toward} equilibrium. The threshold operates only when the perturbation forces the agent into a dominated regime where the opponent can extract surplus.

\paragraph{Liar's Dice: strategic flexibility vs.\ action type.}
Removing all claims from P0 in Liar's Dice forces challenge-only play. Neither tabular Q-Learning (1-die; post $= -0.032$) nor DQN (2-dice; post $= +0.008$) collapses, because challenging at different points is a strategically contingent decision. However, when we force P0 to play \emph{deterministically} (always the lowest legal action, eliminating timing choices), Q-Learning collapses to $-0.524$ (normalized: $0.24$; $p < 0.0001$)---demonstrating collapse at 24{,}576 information sets. This confirms that the threshold depends on the reach-weighted CAC (Eq.~\ref{eq:cac-w}), not the action-type count: challenge-only play retains high $\text{CAC}_w$ because challenge decisions are contingent on private information, whereas deterministic-lowest play sets $\text{CAC}_w = 0$.

\paragraph{Non-zero-sum domains: degradation without collapse.}
In the cooperative Coordination game (match-the-target), forcing P0 to a single action (CAC $= 0$) degrades team performance ($+1.57 \to +1.44$, $p = 0.001$, $d = -3.8$) but does not produce convergence to the DEA. In the Negotiation game (ultimatum, 11 offer actions), forcing P0 to a single offer (CAC $= 0$) degrades outcomes; crucially, P1's policy shifts toward rejection because it can condition on P0's inability to adapt its offer, but this rejection is bounded---P1 does not converge to unconditional rejection as it would under zero-sum exploitation pressure. Retaining partial flexibility (offers 0--2, CAC $= 3$) reverses the degradation entirely. In contrast to competitive settings where zero contingency produces collapse to the DEA, cooperative and mixed-motive environments exhibit bounded degradation, suggesting that the threshold interacts with the underlying interaction structure: zero-sum dynamics amplify the constraint into catastrophic exploitation, while cooperative dynamics produce performance loss without attractor convergence.

\paragraph{Timing invariance.}
Perturbation at episodes 3k, 10k, 17k yields identical collapse ($-0.926$, $-0.927$, $-0.925$). The DEA is a structural attractor independent of training stage.

\FloatBarrier
\subsection{Dynamics}
\label{sec:dynamics}

\paragraph{Recovery.}
Restoring actions at episode 15k produces full recovery ($\Delta = +0.90$) within four episodes---symmetric with collapse speed. The DEA is a maintained attractor, not a corrupted representation. (See Figure~\ref{fig:recovery} and Table~\ref{tab:recovery}.)

\begin{table}[ht!]
\centering
\caption{Recovery (Kuhn, 5 seeds): three-phase Q-Learning reward.}
\label{tab:recovery}
\begin{tabular}{lcc}
\toprule
Phase & Mean & 95\% CI \\
\midrule
Pre (0--10k)       & $-0.035$ & $[-0.05, -0.03]$ \\
Collapsed (10--15k) & $-0.927$ & $[-0.93, -0.92]$ \\
Recovered (15--25k) & $-0.025$ & $[-0.06, -0.01]$ \\
\bottomrule
\end{tabular}
\end{table}

\begin{figure}[ht!]
\centering
\includegraphics[width=0.85\linewidth]{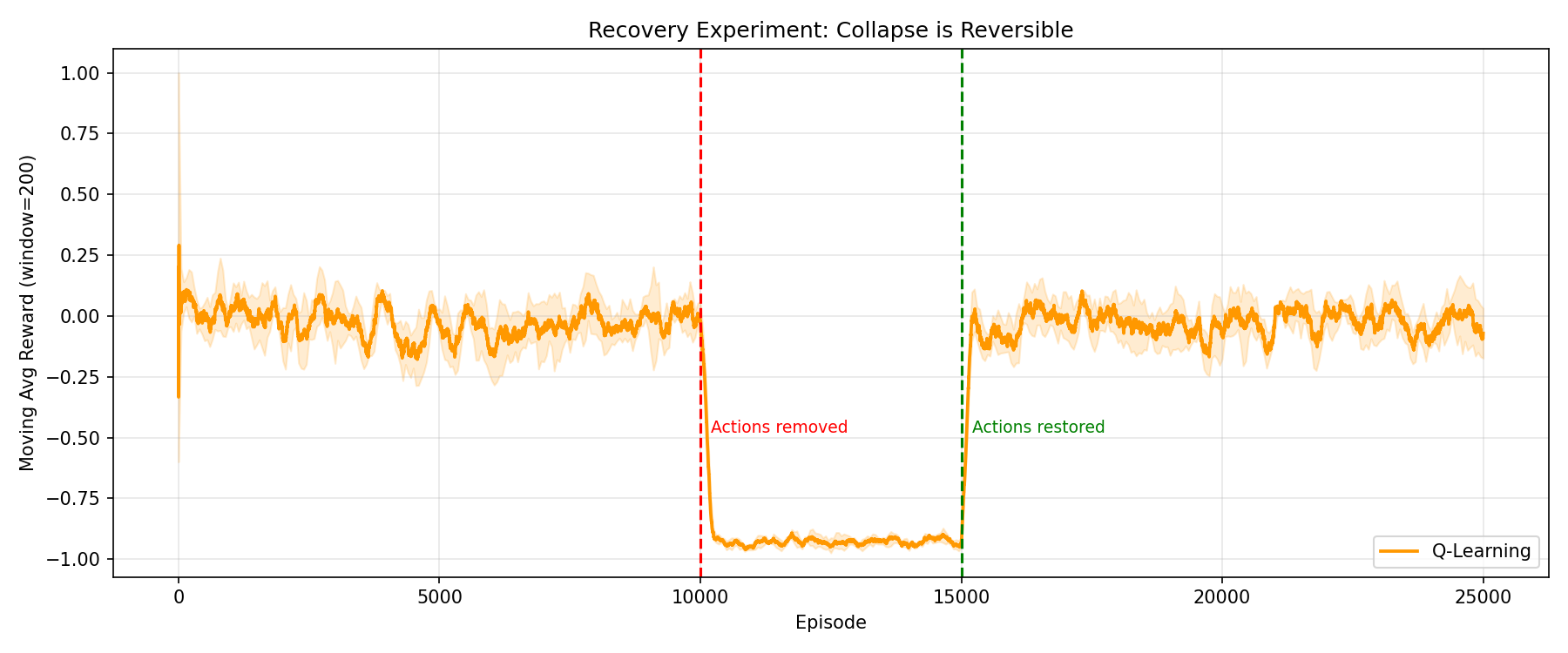}
\caption{Collapse and recovery are symmetric. The DEA is maintained only by the constraint.}
\label{fig:recovery}
\end{figure}

\paragraph{Exploitability trajectory.}
Figure~\ref{fig:exploit} plots exact exploitability (computed via best-response tree walk) over training in Kuhn. Pre-perturbation, exploitability converges toward zero as the agent approaches Nash. Post-perturbation, exploitability spikes. For Leduc Poker, exact exploitability (120 deals, full tree walk) increases from $1.63$ pre-perturbation to $1.93$ post-perturbation ($\Delta = +0.30$), corroborating the reward-based findings with a game-theoretic metric. Reward degradation aligns with increased exploitability where computable (Kuhn, Leduc), supporting generalisation of the DEA characterisation to games where only reward is available.

\begin{figure}[ht!]
\centering
\includegraphics[width=0.85\linewidth]{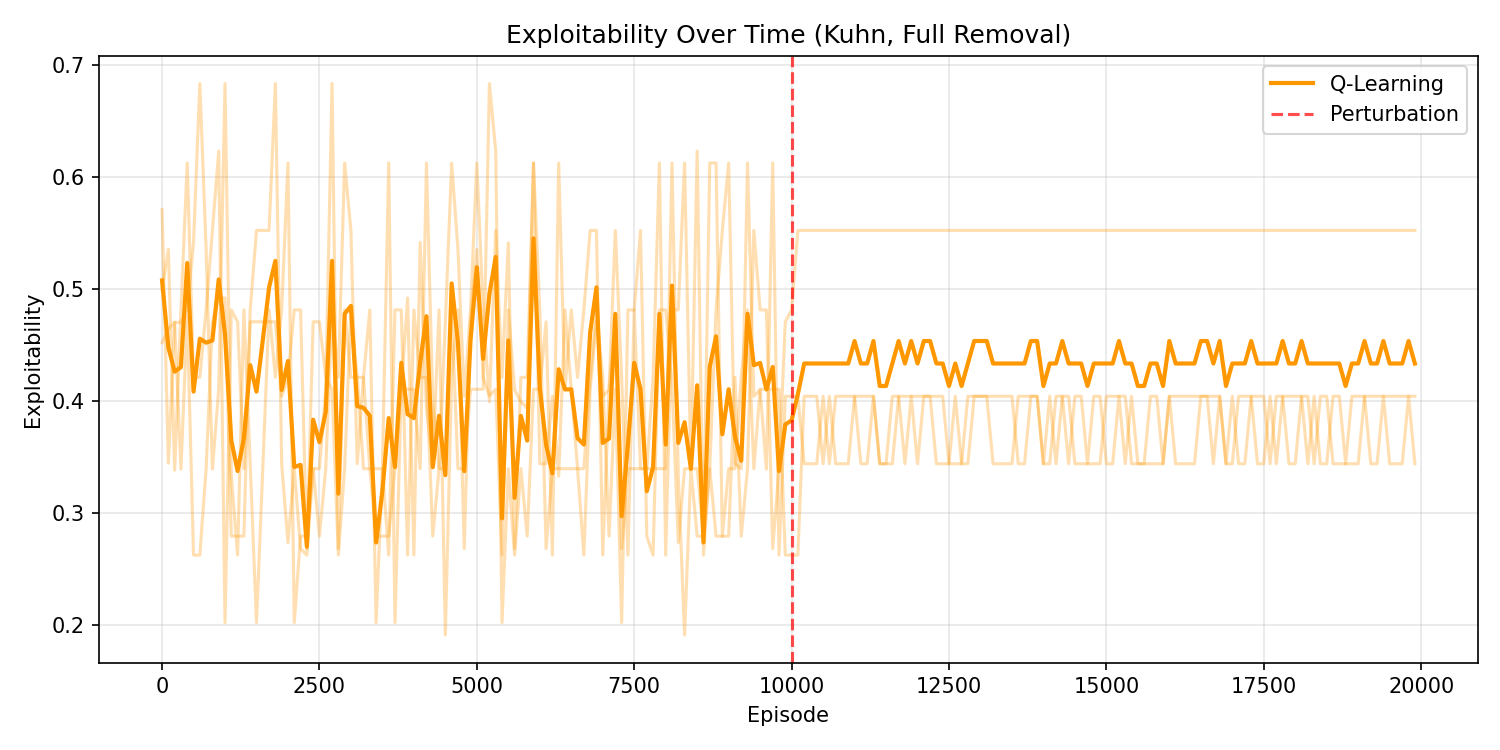}
\caption{Exploitability over time (Kuhn, full removal). Post-perturbation exploitability spikes as the agent's policy becomes deterministically exploitable.}
\label{fig:exploit}
\end{figure}

\paragraph{Variance decomposition.}
Post-collapse reward variance is $5 \times 10^{-6}$, confirming the DEA is a deterministic fixed point.

\FloatBarrier
\section{Conclusion}

We have identified a pronounced threshold in contingent action capacity that governs whether self-play RL agents collapse under asymmetric action-space perturbations. The zero-capacity reduced game is, by construction, a best-response problem against a forced policy; the nontrivial result is that self-play dynamics switch abruptly between a deterministic exploitation attractor at $\text{CAC}_w=0$ and near-Nash stability once positive-reach contingency remains. The severity of the transition depends on the reach-weighted CAC of the retained decision points (Proposition~\ref{prop:residual}): higher reach yields stronger protection. The finding is, within the tested class of discrete, imperfect-information and matrix games, invariant across all tested algorithms (Q-Learning, SARSA, REINFORCE, PPO, NFSP, DQN), co-adaptation-driven, timing-invariant, and fully reversible. A single retained decision point prevents catastrophic exploitation by maintaining strategic coupling between players.

\paragraph{Limitations.}
Games studied range from 1 to $200{,}000{+}$ information sets, with confirmed collapse at 24{,}576 info sets under true zero contingency (Liar's Dice). Experiments report unweighted CAC as a readable proxy for $\text{CAC}_w$; larger games may require direct reach-weighted reporting. Reversibility may not hold under deeper networks where gradient-based collapse could corrupt representations.

\paragraph{Future work.}
Scaling to larger games; cooperative and general-sum settings; formal exploitability bounds beyond Kuhn (approximate BR for Leduc/Liar's Dice); testing regret-minimisation self-play (NFSP) as a dynamic opponent; formalising degradation attractors in non-zero-sum settings.

\bibliographystyle{plainnat}
\bibliography{references}

\clearpage
\appendix

\section{Hyperparameter Sensitivity}
\label{app:hyperparam}
Collapse persists across all 9 tested hyperparameter combinations under zero contingency in Kuhn Poker. Learning rate ($\alpha$) has negligible effect; exploration rate ($\varepsilon$) shifts only the floor (lower $\varepsilon$ = deeper collapse, as expected from reduced residual stochasticity).

\begin{center}
\begin{tabular}{lccc}
\toprule
 & $\alpha = 0.01$ & $\alpha = 0.1$ & $\alpha = 0.3$ \\
\midrule
$\varepsilon = 0.05$ & $-0.974$ & $-0.976$ & $-0.976$ \\
$\varepsilon = 0.15$ & $-0.924$ & $-0.927$ & $-0.927$ \\
$\varepsilon = 0.30$ & $-0.849$ & $-0.853$ & $-0.851$ \\
\bottomrule
\end{tabular}
\end{center}

The DEA is robust to hyperparameter choice. The only variation is the $\varepsilon$-floor effect predicted by Proposition~\ref{prop:zero}.

\section{Separate Self-Play Verification}
\label{app:separate}
Experiments with independent P0/P1 Q-tables produce identical collapse dynamics (separate: $-0.926 \pm 0.003$; shared: $-0.927$; difference: $+0.0008$), confirming that shared-table self-play is equivalent in our zero-sum setting.

\section{Tabular PPO Details}
\label{app:ppo}
Our ``tabular PPO'' maintains a softmax policy over a preference table $\theta(s, a)$ indexed by information state. Action probabilities: $\pi(a|s) = \text{softmax}(\theta(s, \cdot))$. Updates use the clipped surrogate objective $L = \min(r_t A_t, \text{clip}(r_t, 1-\epsilon, 1+\epsilon) A_t)$ where $r_t = \pi_{\text{new}} / \pi_{\text{old}}$, $A_t = R - b$ (advantage with running baseline), and $\epsilon = 0.2$. An entropy bonus $-0.01 \cdot H(\pi)$ is added. Learning rate: $0.01$. The softmax parameterisation explains the partial protection: unlike $\varepsilon$-greedy, softmax policies cannot fully collapse to deterministic play under finite $\theta$ updates, bounding post-perturbation loss at $\approx -0.50$ rather than $-0.93$.

\section{Fixed-Opponent Calibration}
\label{app:fixed-cal}
Under forced check-fold (Kuhn, CAC $= 0$), $P_1$'s best response is to always bet: $P_0$ checks, $P_1$ bets, $P_0$ folds, yielding $+1$ for $P_1$ ($-1$ for $P_0$) on every hand. The $\varepsilon$-floor prevents exact convergence, so the theoretical DEA bound is $v_0 \approx -1 + \varepsilon_{\text{floor}}$---consistent with the observed self-play collapse at $-0.926$.

Under a \emph{fixed Nash opponent} (which does \emph{not} best-respond to $\sigma_0^f$), $P_0$ drops only to $-0.228$. This is far from the BR value of $-1$ because the Nash opponent was trained before perturbation and does not adapt. The gap between $-0.228$ (fixed Nash) and $-0.926$ (self-play) directly quantifies the contribution of co-adaptation: without an adapting opponent, $P_0$ loses only what the Nash strategy extracts from forced check-fold; with co-adaptation, $P_1$ converges toward the BR and extracts nearly the maximum.

\section{Additional Experiments}
\label{app:algorithms}

\paragraph{Severity sweep.}
Perturbation timing ($\times$ severity) under Kuhn Poker. Collapse severity is invariant to timing.

\begin{center}
\begin{tabular}{llcc}
\toprule
Timing & Episode & Severe (CAC 0) & Mild (CAC 1) \\
\midrule
Early & 3{,}000  & $-0.926$ & $-0.063$ \\
Mid   & 10{,}000 & $-0.927$ & $-0.073$ \\
Late  & 17{,}000 & $-0.925$ & $-0.061$ \\
\bottomrule
\end{tabular}
\end{center}

\paragraph{Stochastic masking.}
Root-only perturbation applied with 50\% probability per episode (Kuhn): QL post $= -0.049$ ($p = 0.06$). Intermittent perturbation does not trigger collapse.

\paragraph{Variance decomposition.}
Post-collapse Q-Learning reward variance: $V_{\text{total}} = 5 \times 10^{-6}$, $V_{\text{env}} = 7 \times 10^{-6}$, $V_{\text{policy}} = 6 \times 10^{-6}$. Near-zero total variance confirms the DEA is a deterministic fixed point.

\section{Reach-Weighted CAC Values}
\label{app:cac-w}

\begin{center}
\begin{tabular}{lccc}
\toprule
Perturbation & CAC & $\text{CAC}_w$ & Collapse? \\
\midrule
Full removal (Kuhn) & 0 & 0.000 & Yes ($-0.926$) \\
Root-only (Kuhn) & 1 & 0.473 & No ($-0.065$) \\
No perturbation & 2 & 1.473 & No ($-0.034$) \\
\bottomrule
\end{tabular}
\end{center}

The ``pb'' node has average reach $\rho = 0.473$ under Nash play (P0 passes at root, P1 bets). This positive reach confirms Proposition~\ref{prop:residual}: the retained node has sufficient strategic weight to prevent collapse. The improvement $\delta(h^*) \propto \rho(h^*)$ is directly observable in the data.

\section{DQN Hyperparameters and Normalization}
\label{app:dqn-details}

\begin{center}
\begin{tabular}{ll}
\toprule
Parameter & Value \\
\midrule
Network & 2-layer MLP, 64 hidden units, ReLU \\
Optimizer & Adam, lr $= 10^{-3}$ \\
Replay buffer & 10{,}000 transitions \\
Batch size & 32 \\
Target network update & Every 500 episodes \\
$\varepsilon$ schedule & $0.15 \to 0.01$ linear over 50{,}000 steps \\
\bottomrule
\end{tabular}
\end{center}

\paragraph{Normalization bounds per game.}
\begin{center}
\begin{tabular}{lccc}
\toprule
Game & $r_{\min}$ & $r_{\max}$ & Nash $v_0$ \\
\midrule
Kuhn Poker & $-2$ & $+2$ & $-0.056$ \\
Leduc Poker & $-13$ & $+13$ & $-0.087$ \\
Leduc-4 & $-13$ & $+13$ & $-0.096$ \\
Matching Pennies & $-1$ & $+1$ & $0$ \\
Liar's Dice (1d) & $-1$ & $+1$ & $-0.076$ \\
Liar's Dice (2d) & $-1$ & $+1$ & $\approx 0$ \\
\bottomrule
\end{tabular}
\end{center}

\paragraph{Exploitability computation.} Exact exploitability is computed for Kuhn Poker via full best-response tree walk over all 6 deals: $\text{Exploit}(\pi) = (\text{BR}_{P_0}(\pi) + \text{BR}_{P_1}(\pi)) / 2$, where $\text{BR}_{P_i}$ is the best-response value for player $i$ against $\pi$ playing as the opponent. This is exact (no approximation) for Kuhn's 12 information sets.

\section{Entropy Regularisation}
\label{app:entropy}
We test whether entropy regularisation ($Q(s,a) \leftarrow Q(s,a) + \alpha(r + \tau H(\pi_s) - Q(s,a))$) can prevent collapse. Under zero contingency in Kuhn:

\begin{center}
\begin{tabular}{lc}
\toprule
$\tau$ & QL Post \\
\midrule
$0.00$ (baseline) & $-0.927$ \\
$0.05$ & $-0.925$ \\
$0.10$ & $-0.923$ \\
$0.20$ & $-0.926$ \\
\bottomrule
\end{tabular}
\end{center}

Entropy regularisation has no measurable effect on collapse severity. The DEA is robust to diversity-encouraging bonuses because the zero-contingency constraint eliminates the agent's ability to act on any maintained diversity. The structural nature of the threshold---not insufficient exploration---is what drives collapse.

\section{Reach Sensitivity}
\label{app:reach}

We vary $\varepsilon$ (exploration rate) to change the empirical reach of the retained ``pb'' node under root-only perturbation in Kuhn:

\begin{center}
\begin{tabular}{lcc}
\toprule
$\varepsilon$ & reach(pb) & QL Post \\
\midrule
$0.05$ & $0.464$ & $-0.065$ \\
$0.15$ & $0.469$ & $-0.049$ \\
$0.30$ & $0.486$ & $-0.088$ \\
$0.50$ & $0.517$ & $-0.092$ \\
\bottomrule
\end{tabular}
\end{center}

Post-perturbation reward remains stable near $-0.06$ to $-0.09$ across all reach values, confirming that the retained node provides effective protection once $\text{CAC}_w > 0$, with only modest sensitivity to the exact reach probability.

\end{document}